\documentclass{article}

\usepackage{microtype}
\usepackage{graphicx}
\usepackage{subfigure}
\usepackage{booktabs} % for professional tables
\usepackage{multirow}

\usepackage{hyperref}

\usepackage[accepted]{icml2018}

\usepackage{amsfonts}
\usepackage{amsmath} \allowdisplaybreaks
\usepackage{amsthm}

\newtheorem{lemma}{Lemma}

\usepackage{color}

\begin{document}

\twocolumn[
\icmltitle{Adafactor: Adaptive Learning Rates with Sublinear Memory Cost}

% It is OKAY to include author information, even for blind
% submissions: the style file will automatically remove it for you
% unless you've provided the [accepted] option to the icml2018
% package.

% List of affiliations: The first argument should be a (short)
% identifier you will use later to specify author affiliations
% Academic affiliations should list Department, University, City, Region, Country
% Industry affiliations should list Company, City, Region, Country

% You can specify symbols, otherwise they are numbered in order.
% Ideally, you should not use this facility. Affiliations will be numbered
% in order of appearance and this is the preferred way.
\icmlsetsymbol{equal}{*}

\begin{icmlauthorlist}
\icmlauthor{Noam Shazeer}{google}
\icmlauthor{Mitchell Stern}{google,berkeley}
\end{icmlauthorlist}

\icmlaffiliation{google}{Google Brain, Mountain View, California, USA}
\icmlaffiliation{berkeley}{University of California, Berkeley, California, USA}

%\icmlcorrespondingauthor{Cieua Vvvvv}{c.vvvvv@googol.com}
%\icmlcorrespondingauthor{Eee Pppp}{ep@eden.co.uk}
\icmlcorrespondingauthor{Noam Shazeer}{noam@google.com} % Add a blank corresponding author for the initial submission to avoid an error

% You may provide any keywords that you
% find helpful for describing your paper; these are used to populate
% the "keywords" metadata in the PDF but will not be shown in the document
\icmlkeywords{Machine Learning, ICML}

\vskip 0.3in
]

% this must go after the closing bracket ] following \twocolumn[ ...

% This command actually creates the footnote in the first column
% listing the affiliations and the copyright notice.
% The command takes one argument, which is text to display at the start of the footnote.
% The \icmlEqualContribution command is standard text for equal contribution.
% Remove it (just {}) if you do not need this facility.

\printAffiliationsAndNotice{}  % leave blank if no need to mention equal contribution
% \printAffiliationsAndNotice{\icmlEqualContribution} % otherwise use the standard text.

\begin{abstract}
In several recently proposed stochastic optimization methods (e.g.\ RMSProp, Adam, Adadelta), parameter updates are scaled by the inverse square roots of exponential moving averages of squared past gradients. Maintaining these per-parameter second-moment estimators requires memory equal to the number of parameters.  For the case of neural network weight matrices, we propose maintaining only the per-row and per-column sums of these moving averages, and estimating the per-parameter second moments based on these sums. We demonstrate empirically that this method produces similar results to the baseline. Secondly, we show that adaptive methods can produce larger-than-desired updates when the decay rate of the second moment accumulator is too slow. We propose update clipping and a gradually increasing decay rate scheme as remedies. Combining these methods and dropping momentum, we achieve comparable results to the published Adam regime in training the Transformer model on the WMT 2014 English-German machine translation task, while using very little auxiliary storage in the optimizer.  Finally, we propose scaling the parameter updates based on the scale of the parameters themselves. 

% In addition, we address the problem of the decay rate of these moving averages -- fast decay leads to sub-optimal convergence, while slow decay leaves the model vulnerable to divergence if the gradients suddenly get larger.  We propose slow decay, coupled with a mechanism for up-adjusting the second-moment estimators when the global gradient norm greatly exceeds the moving average.
\end{abstract}

% In addition, we demonstrate that these methods can produce larger-than-desired updates when the decay rate of the second-moment-accumulator is too slow to keep up with changes in the model.  We propose update-clipping as a fix for this behavior.

\section{Introduction and Background}
\label{introduction}

Gradient-based optimization forms the backbone of most modern approaches used to train deep neural networks. One of the simplest methods is stochastic gradient descent (SGD), wherein steps are taken along the direction of the negative gradient of the loss function evaluated on a minibatch. Building on this foundation, a variety of adaptive gradient-based methods have been proposed in which the gradient is divided by the componentwise square root of a vector summarizing the history of squared gradients, usually obtained through summation as in Adagrad \cite{Duchi11Adaptive} or exponential averaging as in RMSProp \cite{Tieleman12RmsProp}, Adam \cite{Kingma14Adam}, and Adadelta \cite{Zeiler12Adadelta}. On convex problems, several of these methods offer theoretical advantages over SGD when gradients are sparse. While convergence guarantees have not yet been provided in the dense, non-convex setting in which most neural network training takes place, practitioners have nevertheless found these methods to empirically outperform SGD across a variety of domains.

The superior performance of these methods does come at a cost. Recent improvements in the \emph{computational capacity} needed to train neural networks with larger numbers of parameters have far outstripped improvements in the \emph{memory capacity} required to store those parameters during training. This has led to memory usage becoming an important constraint on model size. Adaptive optimization algorithms exacerbate this problem by requiring additional memory for extra accumulators, such as those required for momentum and per-coordinate gradient scaling. For example, Adam \cite{Kingma14Adam} keeps two additional values for each parameter, tripling the memory requirements.

We propose a way to reduce memory usage while retaining the empirical benefits of adaptivity by maintaining a factored representation of the squared gradient accumulator across training steps. Specifically, by tracking moving averages of the row and column sums of the squared gradients for matrix-valued variables, we are able to reconstruct a low-rank approximation of the exponentially smoothed accumulator at each training step that is optimal with respect to the generalized Kullback-Leibler divergence. For an $n \times m$ matrix, this reduces the memory requirements from $O(n m)$ to $O(n + m)$. We demonstrate empirically using Adam on a large-scale machine translation task known for its expensive models that our approach achieves comparable performance to that obtained using full accumulators.

%, corresponding to an $O(n)$ factor of savings in the common case of square matrices where $n = m$

Beyond this, we also investigate another issue related to Adam of recent interest. To further reduce memory requirements, we would like to run Adam without momentum, eliminating an additional auxiliary value per model parameter.  But without making any other changes, eliminating momentum can cause training instability.  We identify out-of-date second moment accumulators as a possible cause of this instability and propose two remedies.

Finally, while the learning rate in Adam denotes a target absolute step size, we follow the intuition that relative change in the parameters is more relevant, so we propose scaling the size of the updates relative to the scale of the parameters themselves.

% But without making any other changes, this leads to a large drop in performance.  We analyze two reasons for this discrepancy and propose a possible remedy for each. After thorough testing, we conclude that momentum acts to smooth out sudden spikes in the norm of our update steps, and propose a variant of gradient clipping that we call update clipping to mollify the effect.

\section{A Brief Review of Adam}

\begin{algorithm}
\begin{algorithmic}[1]
\STATE \textbf{Inputs:} initial point $x_0$, step sizes $\{\alpha_t\}_{t=1}^T$, first moment decay $\beta_1$, second moment decay $\beta_2$, regularization constant $\epsilon$
\STATE Initialize $m_0 = 0$ and $v_0 = 0$
\FOR{$t = 1$ \TO $T$}
  \STATE $g_t = \nabla f_t(x_{t-1})$
  \STATE $m_t = \beta_1 m_{t-1} + (1 - \beta_1) g_t$
  \STATE $v_t = \beta_2 v_{t-1} + (1 - \beta_2) g_t^2$
  \STATE $\hat{m}_t = m_t / (1 - \beta_1^t)$
  \STATE $\hat{v}_t = v_t / (1 - \beta_2^t)$
  \STATE $x_t = x_{t-1} - \alpha_t \hat{m}_t / (\sqrt{\hat{v}_t} + \epsilon)$
\ENDFOR
\end{algorithmic}
\caption{Adam \cite{Kingma14Adam}}
\label{alg:adam}
\end{algorithm}

We reproduce the pseudocode for the Adam optimizer in Algorithm~\ref{alg:adam} for reference \cite{Kingma14Adam}. The setup of the problem is as follows. Suppose we are trying to minimize the expected value of a noisy objective function $f(x)$. At each step, we receive a stochastic realization $f_t$, e.g.\ the loss computed on a random minibatch of data, and we compute the gradient $g_t$ of this function with respect to our previous parameters. We then update the exponential running averages of the first and second moments of the gradient $m_t$ and $v_t$, compute bias-corrected versions $\hat{m}_t$ and $\hat{v}_t$ to account for the zero initialization, and finally make a parameter update to obtain a new iterate $x_t$. This repeats for $T$ steps, at which point we return the final iterate $x_T$ as our approximate solution.

The step size $\alpha_t$ is often held constant over the course of training, but recent work in large-scale optimization suggests that performance can be improved on some problems through a linear ramp-up followed by some form of decay \cite{Goyal17Accurate,Vaswani17Attention}. We use the latter with an inverse square root decay scheme in our experiments, finding it to yield more stable results.

%The second moment estimator in (Adam, etc.) is computed as ...

%We achieve updates with a desired step size using the update rule:
%  parameter += grad * (desired step size) / (second moment estimator)

%The desired step size can be constant, or varied according to a schedule - e.g. linear ramp-up followed by inverse-sqrt decay.

\section{Factored Second Moment Estimation}
\label{sec:fsme}

Recent work has shown that for problems where vast quantities of data are available, e.g.\ language modeling and machine translation, task performance improves consistently as model size increases, even in the regime of models with several billions of parameters \cite{Shazeer17Outrageously}. As models continue to grow, the storage requirements of one or two auxiliary parameters per model parameter imposed by existing adaptive methods can be prohibitive, motivating the investigation of a low-memory alternative. In this section, we propose a novel approach in which model structure is exploited in order to reduce storage requirements without compromising empirical performance.

Suppose a subset of the model's parameters are arranged in a matrix, e.g.\ for use in a linear transformation. We denote this subset by $W \subseteq x$ with $W \in \mathbb{R}^{n \times m}$. Under standard practice, we would need to maintain an exponential moving average $V \in \mathbb{R}^{n \times m}$ of the corresponding square gradients $(\nabla_W f(x))^2$ for use in an adaptive update rule.

In cases where storing the full moving average is infeasible, we might instead seek to store moving averages of some low-rank matrices $R \in \mathbb{R}^{n \times k}$ and $S \in \mathbb{R}^{k \times m}$ with $k \ll n, m$ such that $V \approx R S$ at each step. We note that in general, moving averages of instantaneous factors of $V$ may differ from instantaneous factors of the moving average, so standard techniques for low-rank approximation may not necessarily be applicable. We would also like these quantities to be fast to compute so that the approximation step does not become a bottleneck in the overall training procedure.

One common choice for low-rank approximation is to truncate the singular value decomposition at the top $k$ singular values. This is known to give the optimal projection onto the space of rank-$k$ matrices with respect to the Frobenius norm \cite{Eckart36Approximation}. While heavily tuned procedures exist for finding the top $k$ singular values and vectors of a matrix, these quantities in general do not decompose over matrix addition, implying an incompatibility with exponential smoothing. Moreover, there is no guarantee that the entries of the approximation will be nonnegative, which is problematic given that we would like to scale the gradient by the componentwise inverse square root.

In search of a more suitable alternative, we turn to techniques from nonnegative matrix factorization. In addition to the Frobenius norm, another popular cost function in the literature is the generalized Kullback-Leibler divergence, also known as the I-divergence \cite{Lee99Learning}. For nonnegative scalar inputs, the I-divergence is given by the equation $$d(p, q) = p \log \frac{p}{q} - p + q,$$ with the conventions that $0/0 = 0$, $0 \log 0 = 0$, and $p / 0 = \infty$ for $p > 0$. It is easily seen that $d(p, q) \ge 0$ with equality iff $p = q$ by setting $x = p/q$ in the standard inequality $x \log x \ge x - 1$. Under this cost function, we aim to minimize the total elementwise divergence subject to componentwise nonnegativity constraints:
\begin{align}
\begin{split}
\underset{R \in \mathbb{R}^{n \times k}, S \in \mathbb{R}^{k \times m}}{\text{minimize}} \quad & \sum_{i=1}^n \sum_{j=1}^m d(V_{ij}, [RS]_{ij}) \\
\text{subject to} \quad & R_{ij} \ge 0, S_{ij} \ge 0.
\end{split}
\label{eq:factorization-problem}
\end{align}
Solving this problem for general rank-$k$ factors is nontrivial, requiring for instance the use of an alternating minimization procedure \cite{Finesso06Nonnegative}. In the special case of rank-1 factors, however, we can derive an analytic solution.

\begin{lemma}
The solution set of the optimization problem \eqref{eq:factorization-problem} when $k = 1$ consists of all feasible pairs $(R, S)$ satisfying $R S = V 1_m 1_n^\top V / 1_n^\top V 1_m$, where $1_\ell = (1, \dots, 1) \in \mathbb{R}^\ell$ denotes a column vector of $\ell$ ones.
\end{lemma}

\begin{proof}
Let $R$ and $S$ be any feasible solution. Noting that $[RS]_{ij} = R_i S_j$ and expanding the loss, we have
\begin{align*}
& \sum_{i=1}^n \sum_{j=1}^m d(V_{ij}, [RS]_{ij}) \\
= \; & \sum_{i=1}^n \sum_{j=1}^m \left( V_{ij} \log \frac{V_{ij}}{[RS]_{ij}} - V_{ij} + [RS]_{ij} \right) \\
= \; & \sum_{i=1}^n \sum_{j=1}^m V_{ij} \log V_{ij} - \sum_{i=1}^n \sum_{j=1}^m V_{ij} \log R_i \\
& - \sum_{i=1}^n \sum_{j=1}^m V_{ij} \log S_j - \sum_{i=1}^n \sum_{j=1}^m V_{ij} + \sum_{i=1}^n \sum_{j=1}^m R_i S_j .
\end{align*}
Setting the derivatives of this expression with respect to $R_i$ and $S_j$ equal to 0, we obtain the relations
\begin{align*}
- \sum_{j=1}^m \frac{V_{ij}}{R_i} + \sum_{j=1}^m S_j & = 0 \implies R_i = \frac{\sum_{j=1}^m V_{ij}}{\sum_{j=1}^m S_j} , \\
- \sum_{i=1}^n \frac{V_{ij}}{S_j} + \sum_{i=1}^n R_i & = 0 \implies S_j = \frac{\sum_{i=1}^n V_{ij}}{\sum_{i=1}^n R_i} .
\end{align*}
Now note that for any minimizer $(R, S)$, the solution $(\alpha R, S / \alpha)$ is also a minimizer for any $\alpha > 0$, since the loss only depends on the product $R S$. Hence we may break the symmetry by fixing the sum of the components of $R$ at $\sum_{i=1}^n R_i = \sum_{i=1}^n \sum_{j=1}^m V_{ij}$, in which case we obtain a canonical minimizer $$R_i = \sum_{j=1}^m V_{ij}, \quad S_j = \frac{\sum_{i=1}^n V_{ij}}{\sum_{i=1}^n \sum_{j=1}^m V_{ij}}$$ or in vector form, $$R = V 1_m, \quad S = \frac{1_n^\top V}{1_n^\top V 1_m}.$$ By our discussion of symmetry above, it follows that the solution set consists more broadly of all pairs $(R, S)$ satisfying $R S = V 1_m 1_n^\top V / 1_n^\top V 1_m$, and the claim follows.
\end{proof}

We now note some important properties of this rank-1 projection. First, if $V$ itself is a rank-1 matrix, then it will be exactly recovered as one would expect. Second, the projection can be expressed entirely in terms of the row sums $V 1_m$ and column sums $1_n^\top V$, which in particular are linear functions of $V$. This convenient fact gives us the desired compatibility with exponential smoothing, since the row sums of the moving average equal the moving average of the row sums, and similarly for columns. Moreover, storing only the moving averages of these factors rather than the full matrix $V$ yields considerable memory savings, requiring space proportional to $n + m$ rather than $n m$.

\begin{algorithm}[t]
\begin{algorithmic}[1]
\STATE \textbf{Inputs:} initial point $X_0 \in \mathbb{R}^{n \times m}$, step sizes $\{\alpha_t\}_{t=1}^T$, second moment decay $\beta_2$, regularization constant $\epsilon$
\STATE Initialize $R_0 = 0$ and $C_0 = 0$
\FOR{$t = 1$ \TO $T$}
  \STATE $G_t = \nabla f_t(X_{t-1})$
  \STATE $R_t = \beta_2 R_{t-1} + (1 - \beta_2) (G_t^2) 1_m$
  \STATE $C_t = \beta_2 C_{t-1} + (1 - \beta_2) 1_n^\top (G_t^2)$
  \STATE $\hat{V}_t = (R_t C_t / 1_n^\top R_t) / (1 - \beta_2^t)$
  \STATE $X_t = X_{t-1} - \alpha_t G_t / (\sqrt{\hat{V}_t} + \epsilon)$
\ENDFOR
\end{algorithmic}
\caption{Adam for a matrix parameter $X$ with factored second moments and first moment decay parameter $\beta_1 = 0$.}
\label{alg:adam-factored-matrix}
\end{algorithm}

We present a concrete implementation of Adam with factored second moment accumulators in Algorithm~\ref{alg:adam-factored-matrix} for the case where the parameter set $x$ can be viewed as a single matrix $X$. In the event that the parameter set is most suitably partitioned into multiple matrices (treating vectors and scalars as special cases), the steps can be performed in parallel for each matrix individually. We present the algorithm with $\beta_1$ fixed at 0 so as to focus our attention on the second moments. First moments can be included as in Adam without modification if desired.

In the implementation, we keep running averages of the row sums $R_t$ and column sums $C_t$ of the squared gradients. The full accumulator is then approximated as the outer product divided by the sum of all entries, $R_t C_t / 1_n^\top R_t$, and is subsequently scaled by the same bias correction factor as in Adam. We note that the normalization term in the denominator $1_n^\top R_t$ could equivalently be expressed as $C_t 1_m$, so the treatment of row sums and column sums is not asymmetric despite the surface form of the approximation.

\subsection{Relation to Prior Work}

A preliminary version of this method was briefly mentioned in Appendix D of \citet{Shazeer17Outrageously}.  Also, \citet{Gupta14} employ a similar technique, saving memory by averaging Adagrad accumulators across embedding vectors.

\subsection{Experiments}
We ran the Transformer model from \citet{Vaswani17Attention}, using Adam with and without our factored second moment estimation for optimization.  See Section \ref{sec:experiments} for more details on the experimental setup.  Results were similar in all tested cases.  See Table \ref{tab:results} (A) vs.\ (C) and (H) vs.\ (J).  

We also tried simplified estimation schemes where the second-moment estimators for matrices were approximated by either the row means or the column means (but not their outer product).  For this model, the results for the row-mean scheme were similar to baseline, but the results for the column mean scheme were much worse.  See Table \ref{tab:results} (D) and (E).  We suspect that these results are due to the model's use of a shared weight matrix used both to represent the token embeddings and to produce the output probabilities. Each row in this matrix corresponds to one token in the vocabulary. Rows associated with very frequent tokens tend to receive gradients of much larger magnitude than rows associated with very infrequent tokens.

\section{No Momentum}

Adam requires two persistent accumulators per parameter for the first and second moments of the gradients.  In Section~\ref{sec:fsme}, we reduced the memory requirements of the second-moment accumulator.  To remove the need for a first-moment accumulator, we simply turn momentum off by setting $\beta_1=0$.

\subsection{Experiments}

For a step size schedule similar to the one used in \citet{Vaswani17Attention}, which includes a warmup period, model quality is similar without and with momentum (BLEU = 23.6 vs.\ 23.4) -- see Table~\ref{tab:results} (A) vs.\ (B), second to last column.

Without the warmup period, the model without momentum becomes more unstable (BLEU = 0.1 vs.\ 23.1) -- see Table~\ref{tab:results} (A) vs.\ (B), last column.   We hypothesize that removing the momentum unmasks an underlying problem with the stability of Adam, which we will discuss in the next section.

% We observe empirically (as have others) that momentum is good for training stability.   The final column of Table \ref{tab:results} shows experiments following a learning rate schedule similar to the one in \citet{Vaswani17Attention}, using a linear warmup of the step size $alpha_t$ from $0$ to $10^{-4}$ over the first 10,000 steps, followed by inverse-square-root decay.  For this schedule the experiments with and without momentum perform similarly - $BLEU=25.4$ vs.\ $25.6$ - see Table \ref{tab:results} (N) vs.\ (A).   To better test stability, we also ran a sequence of experiments where we eliminated the linear warmup and used constant $alpha_t=10^{-4}$ for the first 10,000 steps, followed by the same inverse-square-root decay.   These results are given in the second-to-last column of Table \ref{tab:results}.  In these experiments, we see much worse results without momentum.  

%For a step size schedule similar to the one used in \citet{Vaswani17Attention}, model quality was similar with and without momentum - See Table \ref{tab:results}, (A) vs (N), last column.   

%We empirically observe (as have others) poorer results without momentum, see Table \ref{tab:results}, lines (A) vs (M).  We hypothesize that removing the momentum unmasks underlying problem with the stability of Adam, which we will discuss in the next section.

\section{A Problem with Adam: Out-of-Date Second Moment Estimator}
\label{sec:msupu}

\citet{Sashank18Convergence} discuss non-convergence issues when using a fast decay of the second-moment estimator (low $\beta_2$).  We observe the same issues in our experiments -- see Table~\ref{tab:slowfast}, first result column.  On the other hand, slow decay (high $\beta_2$) causes training instability when we turn off the step size warmup -- see Table~\ref{tab:slowfast}, second result column.

We explain the instability as follows: A slow decay rate means that our second-moment estimator is based on gradients farther in the past.  If the model is evolving rapidly, this could cause the estimates to have high error, leading to smaller-than-desired or (worse) larger-than-desired updates.  To check whether this is happening, we observe the root-mean-square over all parameters $x$ in a weight matrix or vector $X$ for a given timestep $t$ of the unscaled parameter update $u_{xt} = -g_{xt} / \sqrt{\hat{v}_{xt}}$.  For brevity, we refer to this quantity as $\mathrm{RMS}(U_t)$:
\begin{equation*} %\label{eq:msupu}
\begin{split}
\mathrm{RMS}(U_t) = \mathrm{RMS}_{x \in X}(u_{xt}) = \sqrt{\mathrm{Mean}_{x \in X}\left(\frac{(g_{xt})^2}{\hat{v}_{xt}}\right)} .
\end{split}
\end{equation*}
If Adam is functioning as intended, for each individual parameter $x$, the value $\hat{v}_{xt}$ should be close to $(g_{xt})^2$, since this is precisely what $\hat{v}_{xt}$ is designed to measure.  Thus, the ratio $(g_{xt})^2 / \hat{v}_{xt}$ should be close to 1, as should the mean of many such values.  So for a large weight matrix $X$, a value of $\mathrm{RMS}(U_t)$ which is far from 1 is a sign that the second-moment estimator is not doing its job well.

\begin{table}[t]
\caption{BLEU scores for Transformer machine translation models trained with slow ($\beta_2=0.999$) and fast ($\beta_2=0.9$) second-moment decay, with and without step size warm-up. Fast decay has convergence problems. Slow decay has stability problems.  Excerpted from Table \ref{tab:results} rows (A), (G).}
\label{tab:slowfast}
\vskip 0.15in
\begin{center}
\begin{small}
\begin{tabular}{c|cc}
\toprule
$\beta_2$ & With warm-up & No warm-up \\
\midrule
0.999  & \textbf{25.6} &  0.1\\
0.9  & 18.4  &  15.6 \\
\bottomrule
\end{tabular}
\end{small}
\end{center}
\vskip -0.1in
\end{table}

In Figure \ref{fig:rmsu}, we plot $\mathrm{RMS}(U_t)$ for one particular weight matrix in a Transformer machine translation model \cite{Vaswani17Attention} for training runs with $\beta_2=0.9$ and $\beta_2=0.999$.  With fast decay (red), $\mathrm{RMS}(U_t)$ stays close to $1$ as expected, while with slow decay (blue), it fluctuates significantly.  Values larger than 1 indicate larger-than-desired parameter updates.

The fact that slow decay causes both larger-than-desired updates and training instability supports our hypothesis that the large updates are the cause of the instability, but does not prove it.  One competing hypothesis is that the instability causes the larger-than-desired updates.  We refute this particular competing hypothesis by noting that the $RMS(U_t)$ values plotted in Figure \ref{fig:rmsu} are for training runs with step size warmup, neither of which exhibited instability.  In the next section, we further support our hypothesis by showing that we can cure the instability by clipping the larger-than-desired updates.

% We hypothesize that this may be the reason that the fast-decay model trains better than the slow-decay model (Figure \ref{fig:rmsu} (bottom)).   Comparing these two configurations - Table \ref{tab:results}, rows (A) and (F), the fast-decay model does better for a short training regime with constant learning rate, while the slow-decay model does better for a long training regime with learning rate warmup.  We suspect that the learning rate warmup serves to tone down the larger-than-expected updates in the beginning of the training regime.

\begin{figure}[ht]
%\vskip 0.2in
\centering
\centerline{\includegraphics[width=\columnwidth,trim={0 0 0 1.4cm},clip]{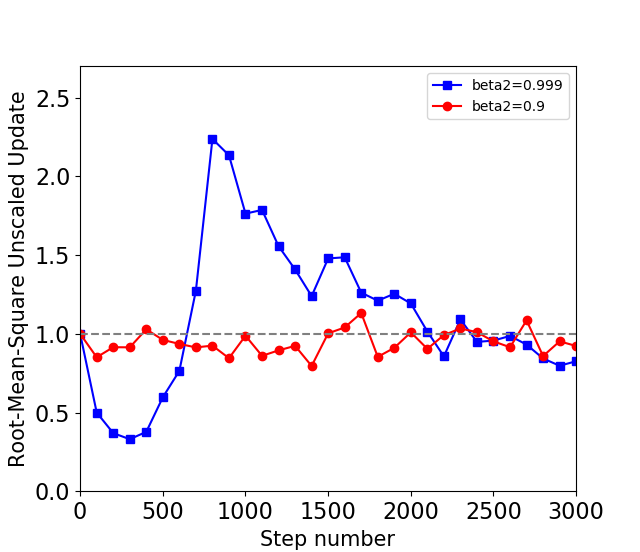}}
\caption{With slow decay $(\beta_2=0.999)$, the second-moment estimator is out of date, leading to parameter updates that are larger or smaller than the intended value.}
\label{fig:rmsu}
%\vskip -0.2in
\end{figure}

\section{Update Clipping}
To remove the larger-than-desired updates described in Section \ref{sec:msupu}, we propose scaling down the updates on a weight vector or matrix $X$ whenever $\mathrm{RMS}(U_t)$ exceeds a threshold value $d$.  We define the clipped unscaled update $\hat{U}_t$ as:
\begin{equation*} \label{eq:clip}
\hat{U}_t = \frac{U_t}{\max\left(1, \mathrm{RMS}(U_t)/d\right)}
\end{equation*}
The actual parameter update is then the product $\alpha_t \hat{U}_t$ of the step size and the clipped unscaled update, as in Algorithm~\ref{alg:adafactor-matrix}.

\subsection{Comparison to Gradient Clipping}
Gradient clipping is a popular heuristic used for training neural networks in which the gradient is scaled down before an update if needed to ensure that its norm never exceeds some fixed threshold \cite{Pascanu13Difficulty}. For stochastic gradient descent, the update direction is exactly the gradient, so this also has the effect of putting an upper bound on the distance traveled in each step.  While gradient clipping is also applied to adaptive methods in practice, the norm of the update direction may still exceed the user-imposed threshold due to the presence of additional per-parameter scaling factors.  In update clipping, we cap the norm of the actual update rather than just the gradient.

\subsection{Experiments}
We added update clipping to the previously described fast-decay experiments.  For the experiment without learning rate warmup, update clipping with $d=1$ significantly ameliorated the instability problem -- see Table \ref{tab:results} (A) vs.\ (H).  With $d=2$, the instability was not improved.  Update clipping did not significantly affect the experiments with warmup (with no instability problems).

% this needs to be justified  with experiments:
% , finding that this empirically provides better stability during training. 

%Previous section \ref{sec:msupu} introduces the mean squared unscaled parameter update $MSUPU(x, t)$ as an indicator of out-of-date second-moment estimators.  Here, we propose using it as a remedy.  If $MSUPU(X, t)$ exceeds a threshold value 

%The quantity which we described in the previous section, the mean across many variables of $/frac{grad^2}{v}$, happens to equal the mean square parameter update.  

% M: Please see the existing section on update clipping below! You may want to borrow something from it.
% N: doing that now....  I will incorporate it.

\section{Increasing Decay Parameter}

An alternative solution to the problems described in Section~\ref{sec:msupu} is to use an increasing schedule of $\beta_2$, as proposed by \citet{Sashank18Convergence}.  Perhaps this can give us the best of both worlds -- see Table \ref{tab:slowfast}, where different decay rates are better in different situations.

\subsection{In Adam}

We point out here that Adam already uses an increasing decay parameter if we rewrite the bias correction as a correction to $\beta_2$.  To do this, we define $\hat{\beta_2}_t = \beta_2 \frac{1 - \beta_2^{t-1}}{1 - \beta_2^t}$, and we compute $\hat{v}_t$ directly in terms of $\hat{v}_{t-1}$ as follows:
\begin{align*}
%v_t & = \beta_2 v_{t-1} + (1 - \beta_2) g_t^2 \\
%\label{eq:correction}
\hat{v}_t & = \frac{v_t}{1 - \beta_2^t} = \frac{\beta_2 v_{t-1} + (1 - \beta_2) g_t^2}{1 - \beta_2^t} \\
& = \frac{\beta_2(1 - \beta_2^{t-1})}{1 - \beta_2^t}\hat{v}_{t-1}  + \frac{1 - \beta_2}{1 - \beta_2^t} g_t^2  \\
& = \hat{\beta_2}_t\hat{v}_{t-1}  + \frac{(1 - \beta_2^t) - (\beta_2 - \beta_2^t)}{1 - \beta_2^t} g_t^2  \\
& = \hat{\beta_2}_t\hat{v}_{t-1}  + \left(1 - \frac{\beta_2(1 - \beta_2^{t-1})}{1 - \beta_2^t}\right) g_t^2  \\
& = \hat{\beta_2}_t \hat{v}_{t-1} + (1 - \hat{\beta_2}_t) g_t^2  .
% \hat{\beta_2}_t & = \beta_2 \frac{1 - \beta_2^{t-1}}{1 - \beta_2^t}
\end{align*}
This, along with similar logic for $\beta_1$, leads to the alternative formulation of Adam in Algorithm~\ref{alg:adam_equivalent}.

\begin{algorithm}
\begin{algorithmic}[1]
\STATE \textbf{Inputs:} initial point $x_0$, step sizes $\{\alpha_t\}_{t=1}^T$, first moment decay $\beta_1$, second moment decay $\beta_2$, regularization constant $\epsilon$
\FOR{$t = 1$ \TO $T$}
  \STATE $g_t = \nabla f_t(x_{t-1})$
  \STATE $\hat{\beta_1}_t = \beta_1 \frac{1 - \beta_1^{t-1}}{1 - \beta_1^t}$
  \STATE $\hat{\beta_2}_t = \beta_2 \frac{1 - \beta_2^{t-1}}{1 - \beta_2^t}$
  \STATE $\hat{m}_t = \hat{\beta_1}_t \hat{m}_{t-1} + (1 - \hat{\beta_1}_t) g_t$
  \STATE $\hat{v}_t = \hat{\beta_2}_t \hat{v}_{t-1} + (1 - \hat{\beta_2}_t) g_t^2$
  \STATE $x_t = x_{t-1} - \alpha_t \hat{m}_t / (\sqrt{\hat{v}_t} + \epsilon)$
\ENDFOR
\end{algorithmic}
\caption{Equivalent formulation of Adam where bias adjustments have been replaced by decay-rate adjustments.}
\label{alg:adam_equivalent}
\end{algorithm}

%It has been observed in the literature that larger values of the second moment decay parameter $0 \le \beta_2 < 1$ in Adam tend to produce higher quality results, but may also introduce instability in the training process \cite{}. Orthogonally, it is also of note that Adam requires a multiplicative correction factor in its updates to account for the bias introduced by zero-initialized accumulators. We argue that both of these issues can be addressed through the use of a decay parameter that varies with time instead of one that stays constant across training. 

In our reformulation of Adam, the corrected decay parameter $\hat{\beta_2}_t = \beta_2 \frac{1 - \beta_2^{t-1}}{1 - \beta_2^t}$ starts at $0$ when $t=1$ and asymptotically approaches $\beta_2$ for large values of $t$.  

\subsection{Proposed Alternative}

Alternatively, we propose the family of schedules $$\hat{\beta_2}_t = 1 - \frac{1}{t^c}, \quad t \ge 1$$ parameterized by a scalar $c > 0$ controlling the rate of increase.

By inspection, it is clear that this schedule starts at 0 for $t = 1$ and increases toward 1 as $t$ tends to $\infty$. This allows us to benefit from the stability properties of a low $\hat{\beta_2}_t$ at the start of training while still realizing the gains in performance due to a high $\hat{\beta_2}_t$ as the run progresses.

Less obviously, this schedule also eliminates the need for bias correction. To see why, we begin by expanding the recursive definition of $v_t$ to arrive at $$v_t = \sum_{i=1}^t (1 - \hat{\beta_2}_i) \prod_{j=i+1}^t \hat{\beta_2}_j g_i^2.$$ Taking expectations of both sides, we have
\begin{align*}
\mathbb{E}[v_t]
& = \mathbb{E} \left[ \sum_{i=1}^t (1 - \hat{\beta_2}_i) \prod_{j=i+1}^t \hat{\beta_2}_j g_i^2 \right] \\
& = \sum_{i=1}^t (1 - \hat{\beta_2}_i) \prod_{j=i+1}^t \hat{\beta_2}_j \mathbb{E}[g_i^2] \\
& = \sum_{i=1}^t (1 - \hat{\beta_2}_i) \prod_{j=i+1}^t \hat{\beta_2}_j \mathbb{E}[g_t^2] \\
& \qquad + \sum_{i=1}^t (1 - \hat{\beta_2}_i) \prod_{j=i+1}^t \hat{\beta_2}_j (\mathbb{E}[g_i^2] - \mathbb{E}[g_t^2]).
\end{align*}
We would like the expected moving average $\mathbb{E}[v_t]$ to be as close as possible to the true second moment $\mathbb{E}[g_t^2]$. If we assume as in \citet{Kingma14Adam} that the gradient distribution is stationary or that the errors $\mathbb{E}[g_i^2] - \mathbb{E}[g_t^2]$ are sufficiently small, then it suffices to check that our proposed decay schedule satisfies $$\sum_{i=1}^t (1 - \hat{\beta_2}_i) \prod_{j=i+1}^t \hat{\beta_2}_j = 1,$$ since this would imply $\mathbb{E}[v_t]$ and $\mathbb{E}[g_t^2]$ are equal in the stationary case or equal up to a small error term otherwise. We will also require that for all $i \ge 1$, $$\lim_{t \to \infty} (1 - \hat{\beta_2}_i) \prod_{j=i+1}^t \hat{\beta_2}_j = 0,$$ which means that the contributions of past gradients will go to 0 as training progresses rather than retaining nontrivial weight for all time.

We verify the first property with a simple induction argument. At time $t=1$, we have $1 - \hat{\beta_2}_1 = 1$ as desired. Then if the equality holds at time $t-1$, we have
\begin{align*}
& \sum_{i=1}^t (1 - \hat{\beta_2}_i) \prod_{j=i+1}^t \hat{\beta_2}_j \\
= \; & \hat{\beta_2}_t \sum_{i=1}^{t-1} (1 - \hat{\beta_2}_i) \prod_{j=i+1}^{t-1} \hat{\beta_2}_j + (1 - \hat{\beta_2}_t) \\
= \; & \hat{\beta_2}_t + (1 - \hat{\beta_2}_t) = 1,
\end{align*}
which completes the argument. We remark that this proof in fact goes through for any schedule for which $\hat{\beta_2}_1 = 0$.

The second condition is more restrictive in comparison. For the proposed schedule, we would like it to be true that
\begin{align*}
& \lim_{t \to \infty} \left( 1 - \left( 1 - \frac{1}{i^c} \right) \right) \prod_{j=i+1}^t \left( 1 - \frac{1}{j^c} \right) \\
& = \frac{1}{i^c} \left( \prod_{j=2}^i \left( 1 - \frac{1}{j^c} \right) \right)^{-1} \lim_{t \to \infty} \prod_{j=2}^t \left( 1 - \frac{1}{j^c} \right) = 0
\end{align*}
for all $i \ge 1$. Using the standard result that for a sequence $0 \le a_n < 1$, the infinite product $\prod_n (1 - a_n)$ converges to a nonzero value iff the series $\sum_n a_n$ converges, we see that the limit above will be 0 iff the series $\sum_{j=2}^\infty 1/j^c$ diverges, which is only true for $c \le 1$. Hence the decay parameter must not increase too fast, as otherwise past gradients will maintain a weight bounded away from 0 for the full duration of training. In the special case where $c=1$, we note that $v_t = \sum_{i=1}^t g_i^2 / t$ reduces to a simple arithmetic moving average of the history of squared gradients.

\begin{table*}[t]
\caption{BLEU scores for Transformer on WMT '14 En $\rightarrow$ De translation task (higher is better).  Each optimization scheme was tested with and without a warmup period.  For the tests with warmup, $s_t = \min(10^{-6}\cdot t, \frac{1}{\sqrt{t}})$.  For the tests without warmup, $s_t = \min(10^{-2}, \frac{1}{\sqrt{t}})$.}  
\label{tab:results}
\begin{center}
\vspace{0.15in}
\scalebox{1.0}{
\begin{tabular}{cccccc|cc}
&   Factored  &    &  & Update & (Relative) & BLEU & BLEU \\
&    Second-Moment &   $\hat{\beta_1}_t$ & $\hat{\beta_2}_t$  & Clipping & Step& with warmup & no warmup \\
&    Estimation &          &                    &     $d$  & Size      &    &  \\
\hline \hline

(A)&   & 0   & $\beta_2=0.999$ & & \multirow{2}{*}{$\alpha_t=0.1 \cdot s_t$}  & 25.6& 0.1 \\
(B)&  & $0.9$ & $\beta_2=0.999$        & & & 25.4 & 23.1 \\
\hline
(C)& yes             & 0              & $\beta_2=0.999$ &  &  \multirow{3}{*}{$\alpha_t=0.1 \cdot s_t$} & 25.4 & 0.2 \\
(D)& use row-mean    & 0              & $\beta_2=0.999$ & & & 25.2 &  0.3 \\
(E)& use col-mean & 0              & $\beta_2=0.999$ & & & 0.3 & 0.5 \\
\hline
(F) &  & 0     & $\beta_2=0.99$   &  & \multirow{2}{*}{$\alpha_t=0.1 \cdot s_t$} & 25.0 & 0.4 \\
(G)&                & 0             & $\beta_2=0.9$  & & & 18.4 & 15.6 \\
\hline
(H)& & 0              & $\beta_2=0.999$ & $1.0$ & & 25.4 & 21.5 \\
(I)&                & 0              & $\beta_2=0.999$ & $2.0$ & & 25.7 & 0.2 \\
(J)& yes            & 0              & $\beta_2=0.999$ & $1.0$ & & 25.6 & 22.4 \\
\hline
(K)&             & 0              & $1-t^{-0.5}$ &  &\multirow{4}{*}{$\alpha_t=0.1 \cdot s_t$}  & 25.6 & 21.1 \\
(L)&             & 0              & $1-t^{-0.8}$ &  & & 25.6 & 0.1 \\
(M)&             & 0              & $1-t^{-1.0}$ &  & & 25.4 & 0.1 \\
(N)&             & 0              & $1-t^{-0.8}$ & $1.0$ & & \textbf{25.9} & 22.4 \\
\hline
(O)&yes & 0 & $1-t^{-0.8}$ & $1.0$ & \multirow{2}{*}{$\rho_t=s_t$} & 25.0 & 25.5 \\
(P)&yes & 0.9 & $1-t^{-0.8}$ & $1.0$ & & 24.9 & 25.3 \\
\hline
\multirow{6}{*}{(Q)} & \multicolumn{4}{c}{SGD} & $lr=1 \cdot s_t$  & 0.6 & 0.8\\
& \multicolumn{4}{c}{SGD} &$lr=10 \cdot s_t$ & 8.2 & 9.1\\
& \multicolumn{4}{c}{SGD} &$lr=100 \cdot s_t$ & 22.9 & diverged \\
& \multicolumn{4}{c}{SGD} &$lr=150 \cdot s_t$ & 24.0 & diverged \\
& \multicolumn{4}{c}{SGD} &$lr=200 \cdot s_t$ & 24.3 & diverged \\
& \multicolumn{4}{c}{SGD} &$lr=300 \cdot s_t$ & diverged & diverged\\
&\end{tabular}
\vspace{0.1in}
}
\end{center}
\end{table*}

\subsection{Experiments}
We added this alternative to our experimental baseline -- see Table \ref{tab:results} lines (A) vs.\ (K), (L), (M).  The schedule $\hat{\beta_2}_t=1-t^{-0.5}$ did in fact maintain both stability and convergence.  When combined with update clipping, this method produced similar results to constant high $\beta_2$ with update clipping -- see Table \ref{tab:results} lines (H) vs.\ (N).

\section{Relative Step Size}

Instead of defining the optimization algorithm in terms of absolute step sizes $\{\alpha_t\}_{t=1}^T$, we propose defining the optimization algorithm in terms of relative step sizes $\{\rho_t\}_{t=1}^T$, which get multiplied by the scale of the parameters.  We define the scale of a parameter vector or matrix as the root-mean-square of its components, lower-bounded by a small constant $\epsilon_2$.  The reason for this lower bound is to allow zero-initialized parameters to escape 0.   Combining this with the other proposals in this paper gives the Adafactor algorithm defined in Algorithms \ref{alg:adafactor-matrix} and \ref{alg:adafactor-vector}.  Proposed hyperparameters for Adafactor are listed in Algorithm \ref{alg:adafactor-hp}.  

\begin{algorithm}[t]
\begin{algorithmic}[1]
\STATE \textbf{Inputs:} initial point $X_0 \in \mathbb{R}^{n \times m}$, relative step sizes $\{\rho_t\}_{t=1}^T$, second moment decay $\{\hat{\beta_2}_t\}_{t=1}^T$ such that $\hat{\beta_2}_1=0$, regularization constants $\epsilon_1$ and $\epsilon_2$, clipping threshold $d$

\FOR{$t = 1$ \TO $T$}
  \STATE $\alpha_t = \max\left(\epsilon_2, \mathrm{RMS}(X_{t-1})\right) \rho_t$
  \STATE $G_t = \nabla f_t(X_{t-1})$
  \STATE $R_t = \hat{\beta_2}_t R_{t-1} + (1 - \hat{\beta_2}_t) (G_t^2 + \epsilon_1 1_n 1_m^\top) 1_m$
  \STATE $C_t = \hat{\beta_2}_t C_{t-1} + (1 - \hat{\beta_2}_t) 1_n^\top (G_t^2 + \epsilon_1 1_n 1_m^\top)$
  \STATE $\hat{V}_t = R_t C_t / 1_n^\top R_t$
  \STATE $U_t = G_t / \sqrt{\hat{V}_t}$
  \STATE $\hat{U}_t = U_t / \max\left(1, \mathrm{RMS}(U_t)/d\right)$
  \STATE $X_t = X_{t-1} - \alpha_t \hat{U}_t$
\ENDFOR
\end{algorithmic}
\caption{Adafactor for weight matrices.}
\label{alg:adafactor-matrix}
\end{algorithm}

\begin{algorithm}[t]
\begin{algorithmic}[1]
\STATE \textbf{Inputs:} initial point $X_0 \in \mathbb{R}^{n}$, relative step sizes $\{\rho_t\}_{t=1}^T$, second moment decay $\{\hat{\beta_2}_t\}_{t=1}^T$ such that $\hat{\beta_2}_1=0$, regularization constants $\epsilon_1$ and $\epsilon_2$, clipping threshold $d$
\FOR{$t = 1$ \TO $T$}
  \STATE $\alpha_t = \max\left(\epsilon_2, \mathrm{RMS}(X_{t-1})\right) \rho_t$
  \STATE $G_t = \nabla f_t(X_{t-1})$
  \STATE $\hat{V}_t =  \hat{\beta_2}_t \hat{V}_{t-1} +  (1 - \hat{\beta_2}_t)  (G_t^2 + \epsilon_11_{n})$
  \STATE $U_t = G_t / \sqrt{\hat{V}_t}$
  \STATE $\hat{U}_t = U_t / \max\left(1, \mathrm{RMS}(U_t)/d\right)$
  \STATE $X_t = X_{t-1} - \alpha_t \hat{U}_t$
\ENDFOR
\end{algorithmic}
\caption{Adafactor for weight vectors.}
\label{alg:adafactor-vector}
\end{algorithm}

\begin{algorithm}[t]
\begin{algorithmic}[1]
\STATE $\epsilon_1 = 10^{-30}$  \\ % epsilon_1 exists just to prevent division by zero 
\STATE $\epsilon_2 = 10^{-3}$ \\
\STATE $d = 1$ \\
\STATE $\rho_t = \min\left(10^{-2}, \frac{1}{\sqrt{t}}\right)$ \\
\STATE $\hat{\beta_2}_t = 1-t^{-0.8}$ \\
\end{algorithmic}
\caption{Proposed hyperparameters for Adafactor}
\label{alg:adafactor-hp}
\end{algorithm}

%TODO(noam): make a box with proposed hyperparameters
%\begin{align*}
%\epsilon_1 & = 10^{-30}  \\ % epsilon_1 exists just to prevent division by zero 
%\epsilon_2 & = 10^{-3} \\
%d & = 1 \\
%\rho_t & = min\left(10^{-2}, \frac{1}{\sqrt{t}}\right) \\
%\hat{\beta_2}_t & = 1-t^{-0.8} \\
%\end{align*}

\subsection{Experiments}
\label{relative_step_size_experiments}

To examine the potential benefit of relative step size,  we use a version of Transformer \cite{Vaswani17Attention} where the token-embedding parameters are not reused in the softmax layer.  The authors cleverly initialize the embedding parameters with standard deviation $\frac{1}{\sqrt{d_\mathrm{model}}}$, similarly to the other parameters, and then scale them up in the computation by a factor of $\sqrt{d_\mathrm{model}}$ so that the embeddings start out with unit norm.  This allows the same absolute step size to work for both the embedding parameters and the other weight matrices in the model.  We test Adam and Adafactor with this ``clever'' embedding scheme, but also with two more naive schemes.   In the first, we initialize the embedding parameters with standard deviation $1$ and do not scale them in the computation.  In the second, we initialize the embedding parameters with standard deviation $\frac{1}{\sqrt{d_\mathrm{model}}}$, and do not scale them in the computation.  For the Adam experiments, we use the hyperparameters and step size scheme from \citet{Vaswani17Attention}.  For the Adafactor experiments, we use our recommended hyperparameters listed in Algorithm \ref{alg:adafactor-hp}.  All models are trained for 50,000 steps with batch size 16,384 tokens (unlike the other experiments in this paper).  Results are given in Table \ref{tab:relative}.  Adafactor proves more resilient to the more naive parameter initialization and scaling schemes.

\begin{table}[t]
\caption{Relative step sizes are more resilient to differently-scaled embedding parameters.}
\label{tab:relative}
\vskip 0.15in
\begin{center}
\begin{small}
\begin{tabular}{cc|cc}
\toprule
Emb init.& Multiplier & BLEU & BLEU \\
$\sigma$ &   & (Adam) & (Adafactor) \\
\midrule
$\frac{1}{\sqrt{d_\mathrm{model}}}$  & $\sqrt{d_{model}}$ & 26.4 & \textbf{26.6} \\
$1$ & $1$ & 25.8 & 26.4 \\
$\frac{1}{\sqrt{d_\mathrm{model}}}$  & $1$  & 24.2 & 25.4 \\
\bottomrule
\end{tabular}
\end{small}
\end{center}
\vskip -0.1in
\end{table}

%\subsection{Update Clipping}

%Gradient clipping is a popular heuristic used for training neural networks in which the gradient is scaled down before an update if needed to ensure that its norm never exceeds some fixed threshold \cite{Pascanu13Difficulty}. For stochastic gradient descent, the update direction is exactly the gradient, so this also has the effect of putting an upper bound on the distance traveled in each step. While gradient clipping is also applied to adaptive methods in practice, the norm of the update direction may still exceed the user-imposed threshold due to the presence of additional per-parameter scaling factors. We propose to cap the norm of the actual update direction rather than just the gradient, finding that this empirically provides better stability during training.

%\citet{Sashank18Convergence} describe a problem where fast decay of the second monent estimator leads to suboptimal convergence.  So we choose slow decay.  This causes another problem - if the gradients get big for one step or for several steps, the second moment estimator takes a long time to adjust.  In the meantime, the udpates are large and the model can diverge.  See experiments.   To prevent these large updates, we apply a per-step adjustment to the second moment estimator based on the ratio of the mean across the entire weight matrix of the second-moment estimator for that step to its moving average.

%of neural network training hardware have enabled the training of models with larger numbers of parameters.  However, improvements in memory capacity have 

\section{Experimental Setup}
\label{sec:experiments}

We evaluated the optimization algorithms described in this paper on the Transformer machine translation model described in \citet{Vaswani17Attention} on the same WMT 2014 English-to-German translation task described in that paper, using the latest version of the architecture from the Tensor2Tensor open-source repository. 

Models were trained for 100,000 steps.  Each training batch contained sentence pairs containing approximately 4,096 tokens in the input and 4,096 tokens in the target sentences.  These batches are about 8 times smaller than the ones used by \citet{Vaswani17Attention}.  This causes our results to be slightly worse, but significantly speeds up training times (less than two hours each on one Google TPU v2). 

In one set of experiments, we followed a similar step size schedule as \citet{Vaswani17Attention} consisting of a linear warmup followed by inverse-square root decay, given by $\alpha_t=0.1\cdot \min(10^{-6}\cdot t, \frac{1}{\sqrt{t}})$.  In order to test training stability, we ran a second set of experiments where the initial warmup was replaced by a flat learning rate: $\alpha_t=0.1\cdot \min(10^{-2}, \frac{1}{\sqrt{t}})$.  For the experiments with relative step sizes, we used schedules $\rho_t=\min(10^{-6}\cdot t, \frac{1}{\sqrt{t}})$ and $\rho_t=\min(10^{-2}, \frac{1}{\sqrt{t}})$.

In addition, we tested plain SGD with learning rate schemes equal to the step size schemes above, multiplied by various constants, since SGD also requires little (zero) additional memory cost.

\subsection{Results}
Results are listed in Table \ref{tab:results}.  The listed BLEU scores are on the development set, newstest2013, using beam search with beam size 4 and length penalty $\alpha=0.6$.  Higher scores are better.  Note that the scores listed should not be compared to those in \citet{Vaswani17Attention}, due to both our shorter training regime and various improvements in the open-source version of the model over the published version.

The schemes with warmup mostly achieved very similar results.  Fast decay of the second-moment estimator (G) was significantly worse.

Without warmup, the baseline (A) becomes unstable.  The instability is relieved by any of momentum (B), fast decay (G), variable decay (K), and gradient clipping (H).   It is not clear whether relative step size has an affect on stability, since the step sizes used in the experiments are not directly comparable. 

Rows (J) and (N) demonstrate algorithms with sub-linear additional memory requirements which attain comparable convergence and stability results to Adam with momentum.

Results for SGD (Q) were poorer and less stable than Adam, and highly dependent on choice of learning rate.

\section{Conclusion}

On a popular machine translation task, we have demonstrated similar quality results to Adam, using a sublinear amount of extra space for accumulators.  This should enable training of significantly larger models on the same memory-constrained hardware.  We have also introduced update clipping, a potentially more-generally-useful technique for stabilizing adaptive gradient methods.

Code for running Adafactor is available in the open-source Tensor2Tensor library.

% Acknowledgements should only appear in the accepted version.
\section*{Acknowledgements}  Thanks to \L{}ukasz Kaiser, the Tensor2Tensor team and~the open-source community for helping test and debug Adafactor.  Also thanks to Geoffrey Hinton, who asserted that training works well if the magnitudes of parameter updates are about $10^{-2}$ to $10^{-3}$ times the magnitude of the parameters.

\bibliography{main}
\bibliographystyle{icml2018}

\end{document}